\documentclass[11pt]{article}

\usepackage[final]{acl}

\usepackage{times}
\usepackage{latexsym}
\usepackage{algorithm}
\usepackage{algpseudocode}

\usepackage[utf8]{inputenc}

\usepackage{microtype}

\usepackage{inconsolata}
\usepackage{siunitx} 

\usepackage{graphicx}
\usepackage{svg}

\usepackage{amsfonts}
\usepackage{amsmath}
\usepackage{bbm}
\usepackage{booktabs}
\usepackage[most]{tcolorbox}
\usepackage{enumitem}
\usepackage[T1]{fontenc} 
\newcommand{\notcheckmark}{\textcolor{gray}{\ding{51}}\textsuperscript{\textcolor{gray}{\kern-0.5em\ding{55}}}}
\definecolor{mayablue}{rgb}{0.45, 0.76, 0.98}

\usepackage{colortbl}
\title{Bidirectional Curriculum Generation: A Multi-Agent Framework for Data-Efficient Mathematical Reasoning}

\author{
  Boren Hu\textsuperscript{1},
  Xiao Liu\textsuperscript{2},
  Boci Peng\textsuperscript{3},
  Xinping Zhao\textsuperscript{4},
  Xiaoran Shang\textsuperscript{5,6},
  Yun Zhu\textsuperscript{6†}, 
  Lijun Wu\textsuperscript{6} \\
  \textsuperscript{1}Zhejiang University, 
  \textsuperscript{2}University of Macau,\\
  \textsuperscript{3}Peking University,
  \textsuperscript{4}Harbin Institute of Technology (Shenzhen),\\
  \textsuperscript{5}Wuhan University,
  \textsuperscript{6}	Shanghai Artificial Intelligence Laboratory\\
  \texttt{boren@zju.edn.cn},\\
  \texttt{zhuyun@pjlab.org.cn}
}

\begin{document}
\maketitle
\renewcommand{\thefootnote}{\fnsymbol{footnote}}
\footnotetext[0]{† Corresponding author.}
\begin{abstract}
Enhancing mathematical reasoning in Large Language Models typically demands massive datasets, yet data efficiency remains a critical bottleneck. While Curriculum Learning attempts to structure this process, standard unidirectional approaches (simple-to-complex) suffer from inefficient sample utilization: they blindly escalate complexity even when foundational gaps persist, leading to wasted computation on unsolvable problems. To maximize the instructional value of every training sample, we introduce a novel Bidirectional Curriculum Generation framework. Unlike rigid trajectories, our multi-agent ecosystem mimics adaptive pedagogy to establish a closed feedback loop. It dynamically generates data by either complicating problems to challenge the model or, crucially, simplying them to repair specific reasoning failures. This mechanism ensures that the model consumes only the most effective data at any given stage. Grounded in the Optimal Pacing Theorem, our approach optimizes the learning trajectory, significantly outperforming baselines while achieving superior reasoning performance with substantially fewer instruction samples.
\end{abstract}

\section{Introduction}
\begin{figure}[t]
    \centering
    \includegraphics[width=\linewidth]{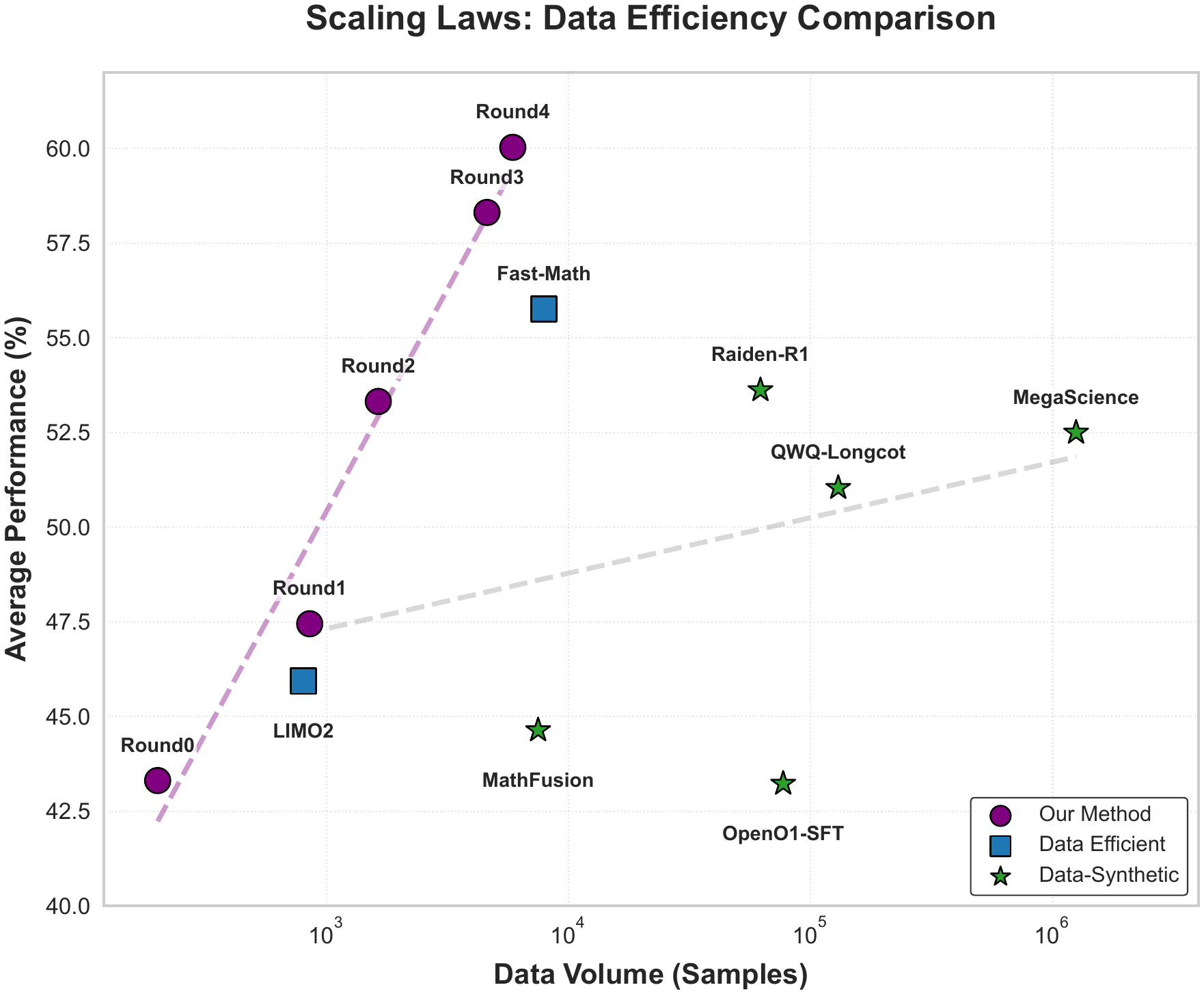}
    \caption{Comparison of mathematical reasoning performance across varying data scales. The x-axis (log-scale) represents the number of training samples, and the y-axis shows the average performance across six benchmarks. The dashed lines represent the fitted scaling laws for our method (purple) and baselines (gray).}
    \label{fig:diversity_pic}
\end{figure}
\label{pic:scalling laws data efficiency}

Data efficiency~\cite{luo2025survey} has emerged as a central challenge in training Large Language Models (LLMs) for mathematical reasoning~\cite{yu2025chain,xu2025towards,liu2025deepseek,yang2025qwen3}. Unlike general text generation, mathematical problem-solving requires strict logical coherence~\cite{sheng2025learning}, prompting a shift from massive scaling to strategic data selection. Recent studies, such as {FastMath}~\cite{yoshihara2025practical} and {LIMO}~\cite{ye2025limo}, demonstrate that data quality can outweigh sheer volume—whether through efficient pipeline optimization or the use of minimal, highly designed cognitive templates.

However, these approaches often face limitations in adaptability. For instance, LIMO relies on a fixed set of expert examples, which lacks the flexibility to address the specific weaknesses of different models. In realistic learning scenarios, effective progress often requires temporarily reverting to intermediate or lower-difficulty examples to repair missing steps when failures occur~\cite{li2024dotamath}. Yet, existing data generation pipelines generally lack mechanisms to diagnose such weaknesses~\cite{wang2023self,luo2023wizardmath}. Furthermore, standard Curriculum Learning (CL)~\cite{wang2021survey,bengio2009curriculum} typically relies on a unidirectional complexity assumption. This open-loop approach can cause training to stagnate or reinforce erroneous reasoning~\cite{lin2024rho} by forcing the model to confront problems for which it lacks the necessary scaffolding.

To address these challenges, we propose a \textbf{Bidirectional Curriculum Generation} framework. Instead of sorting a static dataset, our system utilizes a multi-agent ecosystem to dynamically adjust both problem difficulty and knowledge coverage, creating a closed feedback loop that aligns training data with the model's evolving reasoning abilities.

To enable controllable and interpretable curriculum scheduling, we first introduce a fine-grained \textbf{competition-based difficulty tagging}. We categorize mathematical problems into ten distinct levels, ranging from introductory middle-school problems to the most challenging International Mathematical Olympiad (IMO) tasks. This design provides a continuous difficulty signal, enabling localized diagnosis of model errors and allowing for flexible, non-monotonic transitions between levels.

Built upon this foundation, our framework deploys four collaborative agents that actively construct the optimal learning trajectory:
\begin{itemize}
    \item \textbf{Difficulty-Reduction Agent (The Repairer):} When the model fails, this agent generates transitional examples with reduced constraints to bridge conceptual gaps, preventing the reinforcement of errors.
    \item \textbf{Difficulty-Increasing Agent (The Challenger):} Once performance stabilizes, this agent promotes the model to higher difficulty levels, ensuring the model operates at the frontier of its capability.
    \item \textbf{Reverse-Generation Agent (The Reasoner):} To deepen understanding, this agent reformulates problems by reversing the roles of queries and answers while preserving mathematical equivalence, compelling the model to reason from solutions back to problem conditions.
    \item \textbf{Diversity-Enhancement Agent (The Explorer):} This agent expands coverage across knowledge domains to prevent overfitting to specific problem templates.
\end{itemize}

This approach creates a self-pacing data stream, theoretically supported by the Optimal Pacing Theorem. Extensive experiments demonstrate that our framework achieves superior logical coherence and reasoning performance across multiple benchmarks while significantly improving data efficiency compared to unidirectional baselines.

Our contributions are as follows:
\begin{itemize}
    \item \textbf{Bidirectional Framework:} We propose a dynamic curriculum system that abandons unidirectional scaling in favor of localized, bidirectional difficulty adjustments based on real-time model feedback.
    \item \textbf{Multi-Agent Modulation:} We develop a four-agent ecosystem capable of semantic rewriting, including novel reverse-generation tasks, to robustly train mathematical reasoning.
    \item \textbf{High-Efficiency Training:} Experiments confirm that our method outperforms static baselines while requiring substantially fewer instruction samples, validating the importance of adaptive curriculum generation, as further illustrated by the scaling behavior in Figure~\ref{pic:scalling laws data efficiency}. 
\end{itemize}

\section{Related Works}

The advancement of mathematical reasoning in Large Language Model~\cite{liu2025deepseek,yang2025qwen3} has catalyzed research into structured data construction and adaptive training strategies. Our work lies at the intersection of mathematical reasoning, synthetic data engineering, and curriculum learning.

\subsection{Mathematical Reasoning in Large Language Models} The pursuit of robust mathematical reasoning~\cite{lu2023survey} in LLMs has transitioned from general-purpose scaling to specialized architectural and data-driven optimizations~\cite{guan2025rstarmathsmallllmsmaster}. Foundational techniques such as Chain-of-Thought (CoT)~\cite{wei2022chain} and Tree-of-Thoughts (ToT)~\cite{yao2023tree} have demonstrated the power of structured inference. Recently, specialized models like DeepSeekMath~\cite{shao2025deepseekmath} and MetaMath~\cite{yu2023metamath} have set new benchmarks by fine-tuning on massive, high-quality corpora. Furthermore, the emergence of reasoning-oriented models like DeepSeek-R1~\cite{guo2025deepseek} highlights the significance of large-scale reinforcement learning and verification in mastering complex logic~\cite{zhang2025lessonsdevelopingprocessreward}. While these advancements are significant, the performance of these models remains highly sensitive to the quality and difficulty distribution of the training data, often struggling with "reasoning cliffs"~\cite{khan2025comment} when faced with tasks slightly beyond their training exposure.

\subsection{Scalable Synthetic Data Generation} Due to the scarcity of human-annotated mathematical proofs~\cite{ge2024scaling}, synthetic data generation has become the primary driver of performance gains. {MegaScience}~\cite{yu2023metamath} exemplifies this trend by scaling scientific and mathematical data to unprecedented volumes through fine-grained skill-tree decomposition. In terms of structural synthesis, {MathFusion}~\cite{pei2025mathfusion} introduces instruction fusion techniques to capture relational interdependencies between problems. Additionally, the Evolve-Instruct paradigm, popularized by {WizardMath}~\cite{luo2023wizardmath}, employs LLMs to progressively increase problem complexity~\cite{Greff_2022_CVPR}. However, a common limitation among these methods is their \textit{unidirectional} or \textit{open-loop} nature; they focus on increasing scale or complexity without real-time awareness of the student model's absorption capacity, leading to potential data redundancy or training divergence.


\subsection{Curriculum Learning and Optimal Pacing}

Curriculum Learning suggests that models learn more effectively when samples are presented in an increasing order of difficulty \cite{bengio2009curriculum}. Traditional curriculum learning in LLMs typically follows fixed pacing schedules based on perplexity or sentence length. However, according to the \textbf{Optimal Pacing Theorem}~\cite{hacohen2019power}, the most efficient learning occurs within the \textit{Zone of Proximal Development} (ZPD)~\cite{vygotsky1978mind}—a state where tasks are neither too simple nor overwhelmingly complex.

Unlike previous monotonic complexity scaling, our framework conceptualizes the curriculum as a dynamic, closed-loop process of bidirectional modulation. By adaptively calibrating task difficulty in response to the model's real-time proficiency, this approach ensures the training remains anchored within an optimal learning regime. Such theoretical alignment effectively bypasses the optimization plateaus and reasoning cliffs inherent in static, open-loop data evolution.


\section{Bidirectional Curriculum Generation Framework}
\begin{figure*}[t!]
    \centering
    \includegraphics[width=\linewidth]{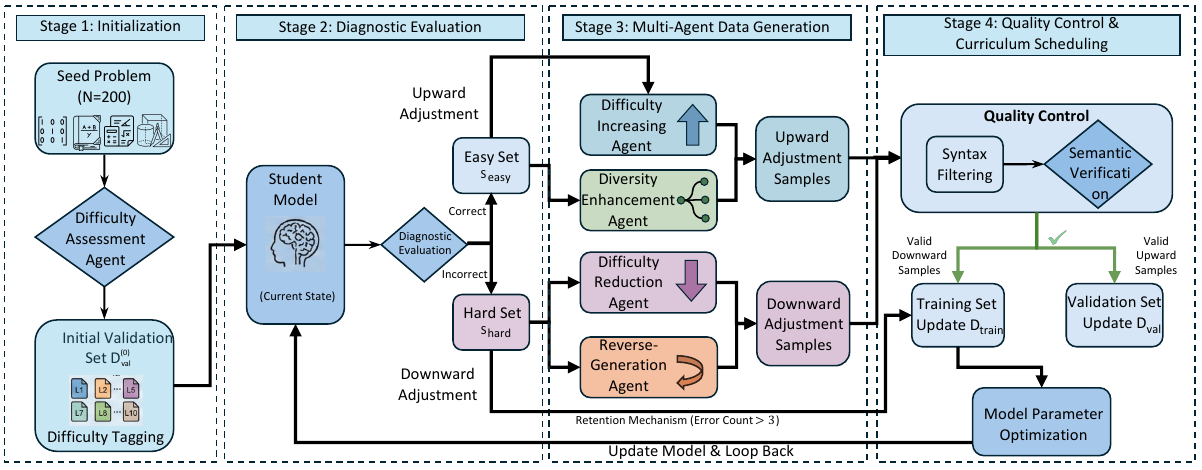}
    \caption{The pipeline of Bidirectional Curriculum with Multi-Agents for Data-Efficient Math Reasoning.}
    \label{fig:overview}
\end{figure*}

To address the critical bottleneck of data inefficiency in mathematical reasoning, we propose the Adaptive Bidirectional Curriculum Generation framework. In this section, we first formalize the problem definition in Section~\ref{subsec:notations}. We then detail the core workflow of our framework, comprising curriculum initialization (Section~\ref{subsec:initialize}), diagnostic evaluation (Section~\ref{subsec:diag}), and the multi-agent generation system that dynamically adjusts problem difficulty (Section~\ref{subsec:multiagent}). Finally, we present the data-model co-evolution mechanism in Section~\ref{subsec:coevolution}. The overall procedure is summarized in Algorithm~\ref{alg:bidirectional_curriculum}, and a theoretical justification of our pacing strategy is provided in Appendix~\ref{sec:appendix}.

\subsection{Problem Definition}\label{subsec:notations}
We formulate the mathematical reasoning task within a problem space $\mathcal{P}$, where each instance $p = (x, y) \in \mathcal{P}$ comprises a problem query $x$ and a reference solution $y$. To quantify complexity, we define a difficulty mapping function $\mathcal{L}: \mathcal{P} \rightarrow \{1, 2, \dots, K\}$, where $K=10$ represents the maximum complexity level\footnote{\url{https://artofproblemsolving.com/wiki/index.php/AoPS_Wiki:Competition_ratings}}. The goal is to train a student model, denoted as $\pi_\theta$, through an iterative curriculum over discrete time rounds $t \in \{1, \dots, T\}$. At each round $t$, the objective is to generate and select an optimal training batch $\mathcal{D}^{(t)} \subset \mathcal{P}$ that aligns with the student's current capability $\pi_{\theta_{t-1}}$. Formally, we aim to minimize the cumulative reasoning error while maximizing data efficiency:
\begin{equation}\min_{\theta} \sum_{t=1}^{T} \mathbb{E}_{(x,y) \sim \mathcal{D}^{(t)}} [-\log \pi_{\theta_t}(y|x)],
\end{equation}
subject to the constraint that $\mathcal{D}^{(t)}$ is synthesized via our Adaptive Bidirectional Curriculum framework. This framework operates as a closed-loop system, consisting of four sequential stages: Seed Initialization, Diagnostic Evaluation, Multi-Agent Data Generation, and Curriculum Co-evolution.

\subsection{Curriculum Initialization}\label{subsec:initialize}
To ensure a high-quality initialization while maximizing data efficiency, we deliberately constrain the diversity of the seed set. Specifically, we use an LLM-as-Judge to assign each candidate problem to one of seven core disciplines: \textit{Prealgebra}, \textit{Algebra}, \textit{Intermediate Algebra}, \textit{Geometry}, \textit{Number Theory}, \textit{Counting \& Probability}, and \textit{Precalculus}:
\begin{equation}
u_i=\Psi(p_i).
\end{equation}
Based on these subject annotations, we perform stratified sampling to build the initial seed dataset $\mathcal{D}_{\text{seed}}=\mathcal{D}^{(0)}_{\text{val}}$, which contains $N=200$ high-quality instances drawn from GSM8K~\cite{cobbe2021gsm8k} and MATH~\cite{hendrycks2021measuring}. As summarized in Table~\ref{tab:data_sources}, this curated subset provides broad topical coverage and a pedagogically coherent starting point for subsequent curriculum construction.

\begin{table}[t!]
\centering

\resizebox{\columnwidth}{!}{
\begin{tabular}{lcr}
\hline
\textbf{Subject Category} & \textbf{Source} & \textbf{Count} \\ \hline
Prealgebra                & GSM8K \& MATH   & 51            \\
Algebra                   & MATH            & 26             \\
Intermediate Algebra      & MATH            & 22             \\
Geometry                  & MATH            & 21             \\
Number Theory             & MATH            & 21             \\
Counting \& Probability   & MATH            & 29             \\
Precalculus               & MATH            & 30             \\ \hline
\textbf{Total}            & -               & \textbf{200}   \\ \hline
\end{tabular}%
}
\caption{Statistics of the seed dataset across different mathematical domains and their respective sources.}
\label{tab:data_sources}
\end{table}

Second, to support the subsequent generation stage, we explicitly annotate the difficulty of each seed problem. Using the same automated evaluator, we define a mapping function $\Phi:\mathcal{P}\rightarrow\mathcal{L}$, where $\mathcal{L}={1,\dots,10}$ denotes our difficulty scale (Table~\ref{tab:diff_tag}). The assigned difficulty level is:
\begin{equation}
\ell_i=\Phi(p_i).
\end{equation}

Finally, we assemble the annotated seed set:
\begin{equation}
\mathcal{D}_{\text{val}}^{(0)}=\{(x_i, y_i, \ell_i, u_i)\}_{i=1}^{N},
\end{equation}
which serves as the starting point for the subsequent pipeline.

\subsection{Diagnostic Evaluation}\label{subsec:diag}
At each curriculum iteration $t$, we conduct a diagnostic evaluation to precisely delineate the capability frontier of the student model $\pi_{\theta_t}$. Specifically, the model generates solutions for all instances in the current validation pool $\mathcal{D}_{\text{val}}^{(t)}$. To ensure a rigorous assessment of the model's operational limits, solution correctness is determined through the joint verification of the generated reasoning chain $\hat{c}$ and the final answer $\hat{y}$. We define the correctness indicator function (see Table~\ref{tab:verification_prompt}) as:
\begin{equation}
    \mathbbm{1}_{\text{corr}}(p, \pi_{\theta_t}) = 
    \begin{cases} 
        1 & \text{if } \hat{c} \text{ and } \hat{y} \text{ are both correct} \\
        0 & \text{otherwise}
    \end{cases}
\end{equation}
Based on this metric, we partition $D_{\text{val}}^{(t)}$ into two disjoint subsets that drive the bidirectional curriculum adaptation:
\begin{equation}
\begin{aligned}
    S_{\text{hard}}^{(t)} &= \left\{ p \in \mathcal{D}_{\text{val}}^{(t)} \mid \mathbbm{1}_{\text{corr}}(p, \pi_{\theta_t}) = 0 \right\} \\
    S_{\text{easy}}^{(t)} &= \left\{ p \in \mathcal{D}_{\text{val}}^{(t)} \mid \mathbbm{1}_{\text{corr}}(p, \pi_{\theta_t}) = 1 \right\}
\end{aligned}
\end{equation}
Here, $S_{\text{hard}}^{(t)}$ captures specific failure modes that necessitate targeted remediation (via simplification), while $S_{\text{easy}}^{(t)}$ represents mastered concepts suitable for complexity scaling (via advancement).

\subsection{Multi-Agent Data Generation}\label{subsec:multiagent}
\label{multi-agent generate data}

To operationalize the curriculum generation, we define a generator ensemble $\mathcal{G} = \{\mathcal{G}_{\text{red}}, \mathcal{G}_{\text{rev}}, \mathcal{G}_{\text{inc}}, \mathcal{G}_{\text{div}}\}$. These specialized agents are designed to address the divergent pedagogical requirements of identifying failures and consolidating mastery. By dynamically synthesizing tailored instruction data, the ensemble ensures that the training distribution shifts adaptively with the student model's proficiency.

\paragraph{Handling Hard Samples: Downward Adjustment.}We define ``downward adjustment'' not merely as simplification, but as a return to foundations to repair reasoning gaps in $S_{\text{hard}}^{(t)}$. This trajectory is executed via two complementary operators. The \textbf{Difficulty-Reduction Agent} ($\mathcal{G}_{\text{red}}$) explicitly lowers the difficulty barrier by reducing problem complexity ($L(p') < L(p)$), offering a smoother learning gradient (see Table~\ref{tab:dif_reduction_prompt}). Complementing this, the \textbf{Reverse-Generation Agent} ($\mathcal{G}_{\text{rev}}$) operates on the logical foundation. By generating inverse problems, it forces the model to revisit the problem's underlying equation from the opposite direction (see Table~\ref{tab:reverse_prompt}). While not necessarily simpler, this operation is ``downward'' in the pedagogical sense: it forces the model to stop advancing and instead deepen its grasp of existing concepts through bidirectional verification, ensuring that the original failure was not due to superficial memorization.

\paragraph{Handling Easy Samples: Upward Expansion.}
For mastered instances in $S_{\text{easy}}^{(t)}$, we initiate an upward expansion trajectory that scales along two complementary axes: reasoning depth and thematic breadth.
First, targeting reasoning depth, the \textbf{Difficulty-Increase Agent} ($\mathcal{G}_{\text{inc}}$) introduces advanced concepts or multi-step dependencies (see Table~\ref{tab:increaser_prompt}). This synthesis yields complex samples $p''$ satisfying $L(p'') > L(p) + \epsilon$ (for $\epsilon > 0$), systematically extending the model's competence frontier.
Second, targeting thematic breadth, the \textbf{Diversity-Enhancement Agent} ($\mathcal{G}_{\text{div}}$) generates structural variants $p_{\text{iso}}$ by re-contextualizing the core logic within diverse target categories (see Table~\ref{tab:add_type_prompt}). While this process alters surface features (such as contextual framing or variable instantiation), it enforces strict difficulty bounds ($|L(p_{\text{iso}}) - L(p)| < \tau$), where $\tau$ represents a predefined tolerance threshold. Empirically, this cross-category perturbation often introduces a slight, beneficial difficulty elevation ($L(p_{\text{iso}}) \ge L(p)$), thereby mitigating overfitting and reinforcing generalization across varied mathematical contexts.

\paragraph{Quality Control and Filtering.} To ensure high data quality, we apply a coarse-to-fine filtering pipeline. {First}, we perform specific formatting checks to filter out malformed samples that lack essential structural tags (e.g., \texttt{<think>} for reasoning traces). {Subsequently}, the surviving samples undergo a content-level assessment via a Verifier Agent (see Table~\ref{tab:verification_prompt}). This agent evaluates reasoning correctness, ensuring that only valid and logically sound samples are retained for training.

\subsection{Curriculum Co-evolution}\label{subsec:coeval}
\label{subsec:coevolution}

\paragraph{Data Evolution.}The transition of datasets involves a strategic allocation of generated samples to either ``instruction'' or ``evaluation'' roles. To construct the curriculum for phase ($t+1$), we enforce three explicit rules:

\textbf{1. Error Retention Policy:} We maintain a failure counter, $\text{ind}_{\text{err}}(p)$, for each problem in the hard subset $S_{\text{hard}}$. Instances exceeding a threshold (i.e., $\text{ind}_{\text{err}}(p) > 3$) are classified as \textit{persistent failures}. To break the stalemate of repeated validation errors, these samples are transferred directly to the training set for supervised memorization.

\textbf{2. Training Set Update (Instructional Scaffolding):} The training set aggregates data intended for direct optimization. We include verified samples from the downward-adjustment agents ($\mathcal{D}_{\text{remedy}}^{(t)}$) because they serve as intermediate scaffolding to bridge specific reasoning gaps.
\begin{equation}
\mathcal{D}_{\text{train}}^{(t+1)} = \mathcal{D}_{\text{remedy}}^{(t)} \cup \{ p \in S_{\text{hard}}^{(t)} \mid \text{ind}_{\text{err}}(p) > 3 \}
\end{equation}
By training on these simplified or inverse problems, the model explicitly learns the foundational logic required to solve the original hard instances.

\textbf{3. Validation Set Renewal (Frontier Expansion):} The validation set is refreshed to establish the next diagnostic baseline. We incorporate advanced samples generated by upward-expansion agents ($\mathcal{D}_{\text{adv}}^{(t)}$) to serve as a new evaluation frontier.
\begin{equation}
\mathcal{D}_{\text{val}}^{(t+1)} = \{ p \in S_{\text{hard}}^{(t)} \mid \text{ind}_{\text{err}}(p) \leq 3 \} \cup \mathcal{D}_{\text{adv}}^{(t)}
\end{equation}
Placing these harder, unmastered samples in the validation set allows us to rigorously assess whether the model's capability boundary has successfully expanded, rather than prematurely training on them.

\paragraph{Model Evolution.}
Once the curriculum $\mathcal{D}_{\text{train}}^{(t+1)}$ is established, the student model parameters $\theta$ are updated via supervised fine-tuning to minimize the cross-entropy loss. 
Upon completion, the updated model $\pi_{\theta_{t+1}}$ advances to the next Diagnostic Evaluation phase using the renewed validation set $\mathcal{D}_{\text{val}}^{(t+1)}$, thereby closing the loop for continuous iterative refinement. 

The integrated workflow, detailing the iterative interaction between diagnostic evaluation, multi-agent generation and model optimization, is presented in Algorithm~\ref{alg:bidirectional_curriculum}.

\begin{table*}[t]
    \centering
    \resizebox{\linewidth}{!}{
    \begin{tabular}{lccccccccc}
        \toprule
        \textbf{Model} & \textbf{Data Volume} & \textbf{GSM8K} & \textbf{MATH500} & \textbf{OMNI-MATH} & \textbf{OlympiadBench} & \textbf{AIME2024} & \textbf{AIME2025} & \textbf{Average} \\
        \midrule
        
        Qwen3-8B-base & - & 92.00 & 79.60 & 30.60 & 47.20 & 6.70 & 10.80 & 44.50 \\
        \midrule
        \rowcolor{gray!10}\multicolumn{9}{l}{\textbf{Data Efficient}} \\

        LIMO2 & 800 & 92.19 & 82.00 & 18.14 & 46.74 & 13.30 & 23.30 & 45.94 \\
        Fast-Math & 7.9K & 92.80 & 86.60 & 39.60 & \textbf{61.00} & 28.80 & 25.80 & 55.76 \\
        \midrule
        \rowcolor{gray!10}\multicolumn{9}{l}{\textbf{Data-Synthetic}} \\
        Raiden-DeepSeek-R1 & 62K & 92.72 & 86.00 & 36.34 & 58.75 & 27.50 & 20.41 & 53.62 \\
        QWQ-LongCoT & 130K & 92.72 & 84.20 & 30.60 & 54.15 & 25.41 & 19.16 & 51.04\\
        OpenO1-SFT & 77K & 91.05 & 76.40 & 25.34	& 44.51	& 12.08 & 10.00 & 43.23 \\
        MegaScience & 1.25M & 93.40 & 84.80 & 35.80 & 57.60 & 25.40 & 17.90 & 52.50 \\
        MATHFusion & 7473 & 92.42 & 79.80 & 27.46 & 48.22 & 6.67 & 13.30 & 44.64 \\
        
        \midrule
         \rowcolor{gray!10}\multicolumn{9}{l}
         {\textbf{Our Method}} \\
        Round 0 (Seed Data) & 200 & 91.05 & 76.80 & 14.57 & 34.12 & 20.00 & 23.33 & 43.31 \\
        Round 1 & 850 & 94.01 & 82.60 & 29.29 & 48.81 & 16.67 & 13.33 & 47.45 \\
        Round 2 & 1,633 & 94.24 & 83.20 & 34.03 & 51.78 & 30.0 & 26.67 & 53.32 \\
        Round 3 & 4,594 & 93.93 & \textbf{89.20} & \textbf{47.42} & 59.35 & 26.67 & 33.33 & 58.31 \\
        Round 4  & 5,873 & \textbf{94.47} & \textbf{89.20} & 46.43 & 60.08 & \textbf{30.00} & \textbf{40.00} & \textbf{60.03} \\
        \bottomrule
    \end{tabular}
    }
    \caption{Comparison of mathematical reasoning performance across different models and data scales. Data Volume denotes the number of samples used for training.}
    \label{tab:main_results}
\end{table*}

\section{Experiments}
\subsection{Experimental Settings}\label{subsec:setting} \paragraph{Benchmarks.} To comprehensively assess mathematical reasoning capabilities, we conduct evaluations across six diverse benchmarks. We utilize GSM8K~\cite{cobbe2021gsm8k} and MATH-500~\cite{lightman2023lets} as \textbf{In-Domain (ID)} benchmarks to measure performance on fundamental and high-school level problems. For \textbf{Out-of-Domain (OOD)} evaluation, we employ AIME 2024~\cite{maa_aime_2024}, AIME 2025~\cite{maa_aime_2025}, Omni-Math~\cite{gao2024omnimath}, and OlympiadBench-Math~\cite{he2024olympiadbench}, which are designed to test the model's proficiency in tackling novel, competition-level challenges. Together, these datasets provide a robust assessment of reasoning across varying levels of complexity and domain specificity. All evaluations are conducted using the OpenDataArena framework~\cite{cai2025opendataarena}.

\paragraph{Baselines.} We compare our method against competitive open-sourced baselines across two paradigms: (1) Data-Efficient methods, including LIMO~\cite{ye2025limo} and FastMATH~\cite{yoshihara2025practical}, which focus on efficiently selecting high-quality subsets from existing datasets; and (2) Data-Synthetic methods, such as MathFusion~\cite{pei2025mathfusion}, MegaScience~\cite{fan2025megascience}, 
QwQ-LongCoT~\cite{amphora2024qwq}, 
OpenO1-SFT~\cite{xia2025generative} and Raiden-DeepSeek-R1~\cite{Raiden-DeepSeek-R1}, which leverage LLMs to synthesize new datasets or rewrite original data.

\paragraph{Model Instantiation and Training.} We utilize DeepSeek~\cite{guo2025deepseek} as the backbone for our four data generation agents; the corresponding prompts are detailed in Appendix \ref{sec:prompts}. For the student model, we adopt Qwen3-8B-Base~\cite{yang2025qwen3}. To balance exploration and exploitation across training iterations, we employ a decaying hyperparameter schedule: the learning rate is gradually reduced from $5 \times 10^{-6}$ to $2 \times 10^{-6}$, and the training duration is decreased from 6 epochs to 3. This progressive strategy allows the model to capture broad foundational knowledge early on while ensuring stable convergence during later refinement. Ultimately, this process yields a total of 5,873 high-quality instances.


\subsection{Main Results}
The main results are demonstrated in Table \ref{tab:main_results}. Our analysis reveals several key insights:

\paragraph{Overall Performance.} Our proposed method significantly outperforms the base model and all open-source baselines after four iterations. With only 5,873 training samples, our model achieves an average score of \textbf{60.03}, representing a substantial improvement of 15.53 over the Qwen3-8B-base and surpassing the strongest baseline, Fast-Math, by 4.27. 

\paragraph{Performance Evolution Across Iterations.} The iterative rounds demonstrate a clear upward trajectory in reasoning capabilities. While the initial Round0 shows a slight performance drop compared to the base model due to the extremely small data scale, the performance improves rapidly from Round1 to Round4. Notably, the most significant leap occurs between Round1 and Round3, where the Average score jumps from 47.45 to 58.31. This suggests that our multi-agent synthesis and filtering pipeline effectively accumulates high-quality reasoning patterns.

\paragraph{Superiority of Synthesis Quality over Data Scale.} Compared to Data-Synthetic methods like MegaScience, which utilizes a massive 1.25M dataset, our method achieves higher performance (60.03 vs. 52.5) using less than 0.5 of its data volume. This contrast highlights that the logical rigor of synthetic data is the primary driver of mathematical reasoning.

\paragraph{Strong Generalization on OOD Tasks.} The most remarkable gains are observed in Out-of-Domain (OOD) competition-level benchmarks. On AIME 2025, our final model reaches \textbf{40.0}, nearly doubling the performance of Raiden-DeepSeek-R1 (20.41) and MegaScience (17.9). These results demonstrate that our iterative refinement process successfully fosters deep reasoning robustness that generalizes well to novel, highly challenging mathematical distributions.

\subsection{Synthetic Data Analysis} 
To evaluate the quality of the synthetic data, we visualize the diversity distribution of the generated samples in Figure~\ref{fig:diversity_pic} and their difficulty distribution in Figure~\ref{fig:difficulty_pic}. Our analysis reveals the following key insights:

\begin{figure}
    \centering
    \includegraphics[width=\linewidth]{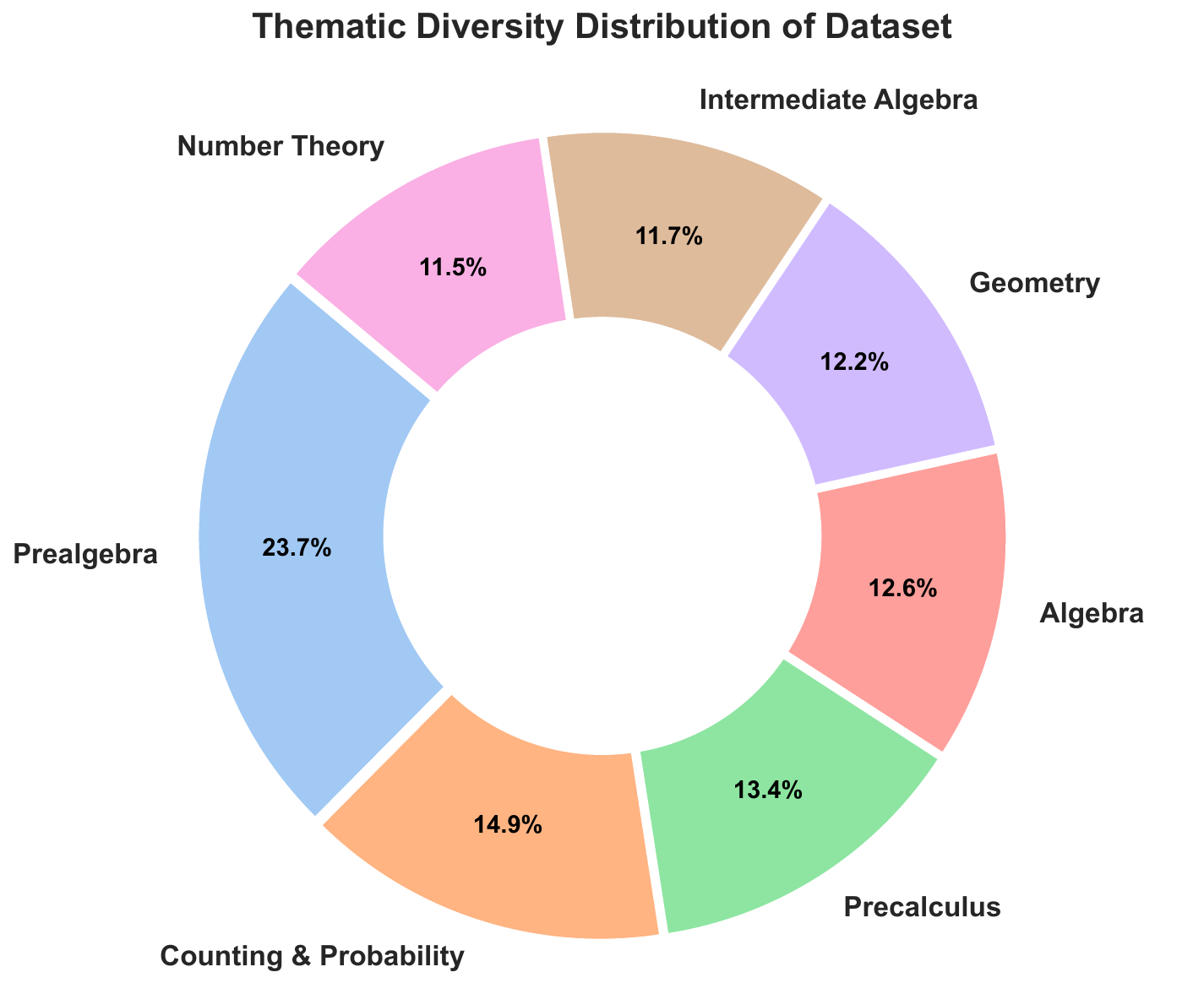}
    \caption{The Diversity distribution about generated datasets}
    \label{fig:diversity_pic}
\end{figure}  
\paragraph{Multi-Domain Semantic Coverage.} The thematic diversity shown in Figure~\ref{fig:diversity_pic} confirms that our multi-agent framework effectively maintains semantic diversity across different domains, which includes seven core mathematical disciplines. Notably, advanced subjects such as \textit{Precalculus} (13.4\%) and \textit{Intermediate Algebra} (11.7\%) are represented with substantial volume, comparable to foundational \textit{Algebra}. This broad coverage ensures that the reasoning heuristics learned by the model are generalizable across diverse mathematical contexts, supporting the strong OOD performance observed in OMNI-MATH and OlympiadBench.

\paragraph{Hierarchical Difficulty Scaling.} As illustrated in Figure~\ref{fig:difficulty_pic}, our dataset maintains a robust pyramidal structure. While foundational levels (1-3) provide necessary basic reasoning patterns, our pipeline successfully retains a significant ``long-tail'' of high-difficulty samples, with over 600 instances categorized in the elite competition tiers (levels 7-10). This high-density concentration of challenging problems directly correlates with the model's breakthrough on the AIME 2025 benchmark, as it forces the model to move beyond simple pattern matching toward deep, multi-round logical deduction.

\begin{figure}
    \centering
    \includegraphics[width=\linewidth]{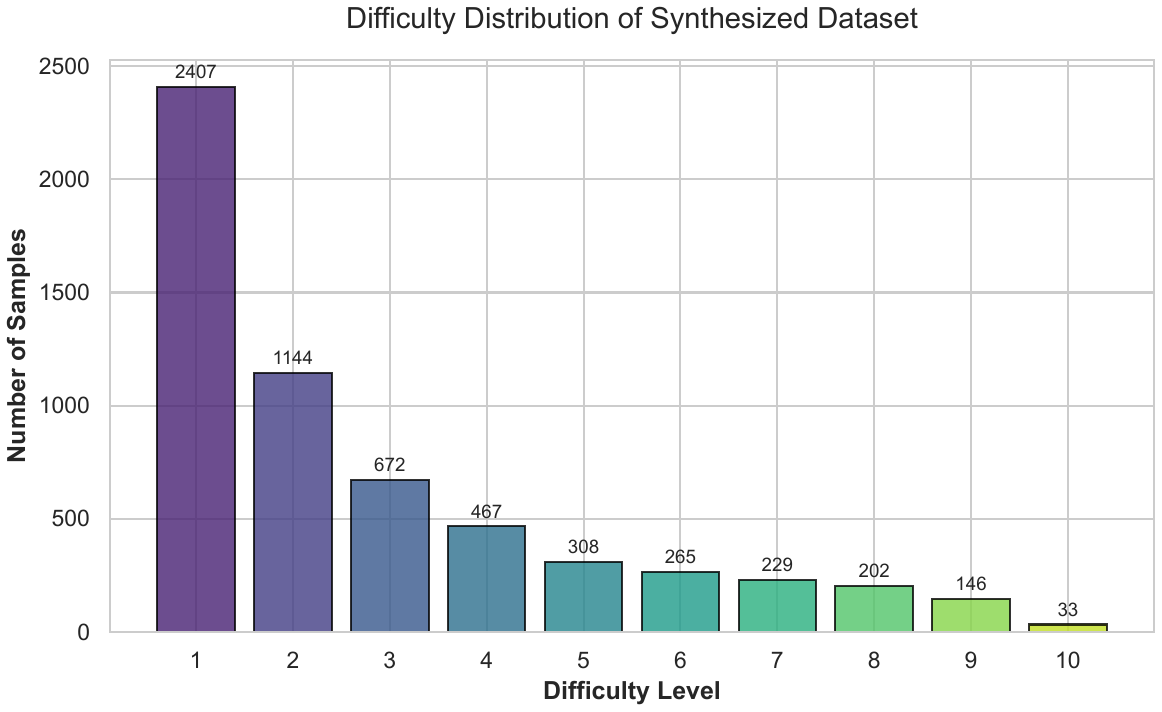}
    \caption{The difficulty distribution of generated datasets}
    \label{fig:difficulty_pic}
\end{figure}

\begin{table*}[t]
    \centering
    \resizebox{\linewidth}{!}{
    \begin{tabular}{lcccccc}
        \toprule
        \textbf{Model} & \textbf{Data Volume} & \textbf{GSM8K} &   \textbf{OlympiadBench} & \textbf{AIME2024} & \textbf{AIME2025} & \textbf{Average} \\
        \midrule
        ALL & 5,873 & 94.47  & 60.08 & 30.00 & 40.00 & 56.13 \\
        \rowcolor{gray!10}\multicolumn{7}{c} {\textbf{Impact of Difficulty Agents}} \\
        Foundational Subset & 4,663 & 93.78  & 61.72 & 23.30 & 33.30 & 53.02 \\
        Advanced Subset & 1,210 & 94.24  & 55.93 & 33.30 & 30.00 & 53.37 \\
        \rowcolor{gray!10}\multicolumn{7}{c} {\textbf{Impact of Reverse Generation Agent}} \\
        w/o Reverse Data & 3,487 & 93.86  & 54.90 & 36.30 & 20.00 & 51.35 \\
        \rowcolor{gray!10}\multicolumn{7}{c} 
        {\textbf{Impact of Diversity Enhancement Agent}} \\
        w/o Algebra & 5,131 & 94.47  & 61.28 & 26.67 & 30.00 & 53.10 \\
        w/o Geometry & 5,158 & 94.62  & 60.98 & 33.33 & 26.67 & 53.89 \\
        w/o Counting & 5,000 & 94.77 & 60.68 & 36.67 & 23.30 & 53,80 \\
        w/o A,C & 4,258 & 95.38 & 61.42 & 26.67 & 23.30 & 51.69 \\
        w/o A, G, C, P & 2,757 &
94.24 &
58.01 &
16.67 &
20.00 &
47.23  \\
        \bottomrule
    \end{tabular}
    }
    \caption{The ablation experiments.}
    \label{tab:ablation_experiments}
    \vspace{-0.16cm}
\end{table*}
\subsection{Ablation Study}
To dissect the efficacy of individual components within our proposed framework, we conducted a comprehensive series of ablation studies including the impact of difficulty agents, the impact of reverse generation agent, and the impact of diversity enhancement agent. In addition, all model variants are fine-tuned using the Qwen3-8B-Base backbone for a consistent duration of 6 epochs. 

\paragraph{Impact of Difficulty Agents.} To validate the necessity of our bidirectional curriculum—specifically the roles of difficulty-reduction agents ($\mathcal{G}_{\text{red}}$) versus difficulty-increase agents ($\mathcal{G}_{\text{inc}}$)—we stratified the curated dataset into two distinct subsets based on the difficulty level $L(p)$. We defined a {Foundational Subset} ($L(p) < 5$, comprising 4,663 samples) representing the output of downward adjustment, and an {Advanced Subset} ($L(p) \ge 5$, comprising 1,210 samples) representing the output of upward expansion. We trained separate models on each subset under identical hyperparameters and compared them against the full model (``ALL''). The comparative results are summarized in Table \ref{tab:ablation_experiments}, yielding the following critical insight: \\The model trained on the full dataset (``ALL'') achieves the highest average performance (56.13), outperforming the Foundational (53.02) and Advanced (53.37) subsets. This validates our core hypothesis: high-performing mathematical reasoning requires both the scaffolding provided by easier problems and the cognitive stretch provided by complex problems. The bidirectional approach prevents the model from being trapped in low-level reasoning or overwhelmed by high-complexity leaps.
\paragraph{Impact of Reverse Generation Agent} To evaluate the contribution of the reverse-generation agent ($\mathcal{G}_{\text{rev}}$), we removed all samples generated through solution-condition inversion, resulting in a training set of 3,487 samples (``w/o Reverse Data''). As shown in Table \ref{tab:ablation_experiments}, removing reverse data leads to a noticeable decline in average performance (dropping from 56.13 to 51.35). As a result, by forcing the model to reconstruct problem conditions from solutions, the reverse agent encourages a deeper, more symmetrical understanding of mathematical relationships, which is essential for generalization to unseen, high-difficulty problems.

\paragraph{Impact of diversity enhancement agent}
We evaluate the Diversity Enhancement Agent by selectively removing mathematical domains. Table \ref{tab:ablation_experiments} indicates that reduced data diversity directly impairs model performance. Specifically, removing any single domain—Algebra, Geometry, or Counting—decreases the average score from the baseline 56.13 to approximately 53.8. This degradation is exacerbated by multi-domain exclusion: removing both Algebra and Counting (``w/o A,C") yields an average of 51.69, while excluding four types (``w/o A,G,C,P") causes a steep decline to 47.23. Crucially, while performance on easier tasks (GSM8K) is unaffected, capability on hard benchmarks drops sharply (e.g., AIME 2024 falls from 30.00 to 16.67), validating the agent's critical role in handling complex reasoning tasks.

\section{Conclusion}
This paper introduces a Bidirectional Curriculum Generation framework to overcome data efficiency bottlenecks in LLM mathematical reasoning. By replacing rigid unidirectional scaling with a multi-agent feedback loop, our system dynamically adjusts problem difficulty and diversity to match the model’s real-time reasoning frontier. Experimental results confirm that our approach significantly outperforms static baselines. Specifically, we achieve superior performance on complex benchmarks like AIME and Omni-Math using substantially fewer instruction samples. This validates the Optimal Pacing Theorem in practice, demonstrating that adaptive, bidirectional data synthesis is key to robust and efficient cognitive training.

\clearpage
\section*{Limitations}
Despite its effectiveness, our framework is optimized for mathematical reasoning, relying on a competition-based difficulty tagging system and domain-specific agents (e.g., Reverse-Generation). Adapting this bidirectional curriculum to less structured fields, such as creative writing or legal reasoning, remains challenging due to the difficulty of defining objective difficulty levels and logical failures in those contexts.




\bibliography{custom}

\appendix

\section{Theorem Proof}
\label{sec:appendix}

\newtheorem{theorem}{Theorem}

\newenvironment{proof}[1][Proof 1]{\par\noindent\textbf{#1} \upshape}{\hfill}

\let\oldtheorem\theorem
\let\endoldtheorem\endtheorem
\makeatletter
\renewenvironment{theorem}[1][]{%
  \oldtheorem[#1]%
  \upshape
}{%
  \endoldtheorem%
}
\makeatother

\begin{theorem}[Optimal Pacing Theorem]
Let the model's capability level at time $t$ be $c_t$, and the sample difficulty be $d$. There exists an optimal difficulty interval $[c_t - \varepsilon, c_t + \varepsilon]$, such that when sampling within this interval, the expected gradient norm of the student model parameters $\theta$ is maximized, and the convergence speed is fastest.
\end{theorem}

\begin{proof}
\textbf{\\  Functional Relationship between Gradient Norm and Difficulty}

Let $L(\theta, x)$ be the loss function for sample $x$. According to the curriculum learning hypothesis \cite{bengio2009curriculum}, the contribution of a sample to model parameter updates (i.e., gradient magnitude) has a non-monotonic relationship with difficulty $d$. We define the gradient function $g(d)$:

\[
\|\nabla_\theta L(\theta, x_d)\| \approx g(d) = A \cdot d \cdot e^{-d / c_t}
\]

where $c_t$ represents the model's capability level at iteration $t$.

\noindent\textbf{When $d \to 0$ (extremely simple sample):} the student model $M_S^{(t)}$ has completely mastered the concept, $L \to 0$ leading to $\nabla_\theta L \to 0$.

\noindent\textbf{When $d \to \infty$ (extremely difficult sample):} the sample resides in the noise region of $M_S^{(t)}$, where predictions are uncorrelated with labels, causing random gradient directions and vanishing magnitudes due to Softmax saturation.

\textbf{Finding the Extremum (Optimal Zone)}

To maximize learning efficiency, we maximize gradient contribution. Taking the first derivative of $g(d)$:

\[
\frac{\partial g}{\partial d} = A \cdot e^{-d / c_t} \left(1 - \frac{d}{c_t}\right)
\]

Setting $\frac{\partial g}{\partial d} = 0$ yields the optimal difficulty $d^* = c_t$.

\textbf{Mapping Effects of Generators $\mathcal{G}_{\text{down}}$ and $\mathcal{G}_{\text{up}}$}

The transformation operators defined in Section \ref{multi-agent generate data} induce the following mappings:

\noindent\textbf{Downward mapping:} For hard samples $p \in S_{\text{hard}}^{(t)}$ with $d(p) > c_t + \varepsilon$, generator $\mathcal{G}_{\text{down}}$ (comprising $\mathcal{G}_{\text{red}}$ and $\mathcal{G}_{\text{rev}}$) produces samples with difficulty:
\[
d(\mathcal{G}_{\text{down}}(p)) = d(p) - \Delta d \in [c_t, c_t + \varepsilon]
\]

\noindent\textbf{Upward mapping:} For easy samples $p \in S_{\text{easy}}^{(t)}$ with $d(p) < c_t - \varepsilon$, generator $\mathcal{G}_{\text{up}}$ (comprising $\mathcal{G}_{\text{inc}}$ and $\mathcal{G}_{\text{div}}$) produces samples with difficulty:
\[
d(\mathcal{G}_{\text{up}}(p)) = d(p) + \Delta d \in [c_t - \varepsilon, c_t]
\]

\textbf{Error Propagation and Risk Bound}

The empirical risk $R_{\text{emp}}$ evolves according to:
\[
R_{\text{emp}}(\theta_{t+1}) \leq R_{\text{emp}}(\theta_t) - \eta \|\nabla_\theta L\|^2 + \mathcal{O}(\eta^2)
\]
where $\eta$ is the learning rate. When the sample difficulty is constrained to $[c_t - \varepsilon, c_t + \varepsilon]$ through the operators, the gradient magnitude $\|\nabla_\theta L\|$ remains at a high level (i.e., $g(d) \geq g(d^*) - \delta$, meaning the gradient norm stays within a small tolerance $\delta > 0$ of its theoretical maximum $g(d^*)$ at optimal difficulty $d^* = c_t$), thus ensuring the maximum decrease of the risk function in each iteration.

\textbf{Conclusion:}

Through dynamic adjustment by $\mathcal{G}_{\text{down}}$ and $\mathcal{G}_{\text{up}}$, the training distribution is concentrated in the peak region of $g(d)$. This avoids gradient vanishing from extremely hard samples while eliminating redundant computations on easy samples, theoretically proving that $[c_t - \varepsilon, c_t + \varepsilon]$ defines the optimal learning pace for student model $M_S$.
\end{proof}
\section{Prompts}
\label{sec:prompts}
The prompts used in our framework are summarized as follows:
\begin{itemize}
    \item \textbf{Difficulty tagging prompt}: Table~\ref{tab:diff_tag}.
    \item \textbf{Difficulty-reduction agent prompt}: Table~\ref{tab:dif_reduction_prompt}.
    \item \textbf{Difficulty-increasing agent prompt}: Table~\ref{tab:increaser_prompt}.
    \item \textbf{Diversity enhancement agent prompt}: Table~\ref{tab:add_type_prompt}.
    \item \textbf{Reverse generation agent prompt}: Table~\ref{tab:reverse_prompt}.
    \item \textbf{Mathematical solution verification agent prompt}: Table~\ref{tab:verification_prompt}.
\end{itemize}

\begin{table*}[t!]
    \centering
    \scriptsize 
    \renewcommand{\arraystretch}{1.0}
    \begin{tcolorbox}[
        sharp corners,
        colback=gray!5,
        colframe=black!80,
        fonttitle=\bfseries,
        title=Prompt for Difficulty Tagging,
        boxrule=0.6pt,
        left=4pt, right=4pt, top=4pt, bottom=4pt, 
        arc=0mm
    ]
        \textbf{\#\# Role} \\
        You are an experienced math olympiad coach responsible for evaluating the difficulty of the given 'math problem'.

        \vspace{0.2em}
        \textbf{\#\# Goal} \\
        For the given 'math problem', you need to evaluate its difficulty level from 1 to 10 according to the 'rule' specified below.

        \vspace{0.2em}
        \textbf{\#\# Rule (Standard 1-10 Scale)} \\
        The following standard represents a progression from Middle School to Junior High, High School, Challenging High School, and finally Olympiad levels.

        \begin{description}[style=unboxed, leftmargin=0pt, font=\bfseries, nosep, itemsep=2pt]
            \item[Level 1:] Problems strictly for beginners, covering elementary or early middle school content. Involves standard techniques and traditional word problems. \textit{e.g., How many integer values of $x$ satisfy $|x| < 3\pi$?}
            
            \item[Level 2:] For motivated beginners. Involves harder instructions, complex word problems, and requires multi-step reasoning beyond basic techniques. \textit{e.g., A fair 6-sided die is repeatedly rolled until an odd number appears. Probability of all evens before first odd?}
            
            \item[Level 3:] Advanced beginner problems requiring creative thinking and synthesis of multiple concepts. \textit{e.g., Triangle $ABC$ with $AB=50, AC=10$ area 120... find Area of quadrilateral $FDBG$.}
            
            \item[Level 4:] Intermediate-level problems involving sequences, functional iteration, or systems of equations. \textit{e.g., Sequence $x_{n+1}=\frac{x_n^2+5x_n+4}{x_n+6}$, find least $m$ s.t. $x_m \leq 4+\frac{1}{2^{20}}$.}
            
            \item[Level 5:] High-intermediate problems, including simple proof-based Olympiad-style questions. \textit{e.g., Find triples $(a, b, c)$ s.t. $a+b+c=\sum \frac{1}{a}$ and $a^2+b^2+c^2=\sum \frac{1}{a^2}$.}
            
            \item[Level 6:] Advanced high-school level to introductory Olympiad level. Requires integration of geometry, algebra, and number theory. \textit{e.g., Acute $\triangle ABC$ with orthocenter $H$... Area is $m\sqrt{n}$.}
            
            \item[Level 7:] Mid-tier Olympiad level. Requires technical knowledge, combinatorial constructions, or graph-theoretic reasoning. \textit{e.g., Prove existence of a balanced set of $n$ points and determine centre-free conditions.}
            
            \item[Level 8:] High-level Olympiad problems. Demands clever constructions or deep insights. \textit{e.g., Splitting coins of denomination $\frac{1}{n}$ with total value $99+\frac{1}{2}$ into groups $\le 1$.}
            
            \item[Level 9:] Expert Olympiad level. Involves advanced number theory, algebraic combinatorics, or functional constraints. \textit{e.g., Primes on a circle where neighbor product is $x^2+x+k$.}
            
            \item[Level 10:] Historically extreme problems. Exceedingly tedious, technical, or rare. Solved by very few worldwide. \textit{e.g., Existence of line separating $n$ points with margin $\ge c n^{-1/3}$.}
        \end{description}

        \vspace{0.2em}
        \textbf{\#\# Output Format (STRICT)} \\
        Return exactly one integer enclosed in \texttt{\textless score\textgreater} tags. No explanation. Example: \texttt{\textless score\textgreater 7 \textless/score\textgreater}

        \vspace{0.2em}
\textbf{\#\#\# Examples} \\
        \textbf{Example (Level 1):} Problem: ... $|x| < 3\pi$? \textbf{Respond with:} \texttt{\textless score\textgreater 1 \textless/score\textgreater} \\
        \textbf{Example (Level 5):} Problem: ... $a+b+c=\sum \frac{1}{a}$ ... \textbf{Respond with:} \texttt{\textless score\textgreater 5 \textless/score\textgreater} \\
        \textbf{Example (Level 10):} Problem: ... at least $c n^{-1/3}$. \textbf{Respond with:} \texttt{\textless score\textgreater 10 \textless/score\textgreater}

        \vspace{0.2em}
        \textbf{Instruction:} \\
        \{instruction\} \\
        \textbf{Respond with:} \texttt{\textless score\textgreater ... \textless/score\textgreater}
    \end{tcolorbox}
    \caption{The full prompt used for difficulty tagging.}
    \label{tab:diff_tag}
\end{table*}

\begin{table*}[ht!]
    \centering
    \scriptsize 
    \begin{tcolorbox}[
        sharp corners,
        colback=white,
        colframe=black!85,
        fonttitle=\bfseries,
        title=Prompt for Difficulty-Reduction Agent,
        boxrule=0.7pt,
        left=5pt, right=5pt, top=5pt, bottom=5pt
    ]
        \textbf{\#\# Task Definition} \\
        You are an expert mathematics problem designer. Based on the input problem, generate a new problem:
        \begin{itemize}[nosep, leftmargin=1.5em]
            \item If original difficulty $D > 1$, generate a new problem at level $D-1$.
            \item If $D = 1$, generate a similar Level 1 problem with different numbers/context for expansion.
        \end{itemize}

        \vspace{0.4em}
        \textbf{\#\# Difficulty Calibration Standard (Levels 1--10)}
        \begin{description}[style=unboxed, leftmargin=0pt, font=\bfseries, nosep, itemsep=2pt]
            \item[Level 1:] Problems strictly for beginners, covering elementary or early middle school content. Involves standard techniques and traditional word problems. \textit{e.g., How many integer values of $x$ satisfy $|x| < 3\pi$?}
            
            \item[Level 2:] For motivated beginners. Involves harder instructions, complex word problems, and requires multi-step reasoning beyond basic techniques. \textit{e.g., A fair 6-sided die is repeatedly rolled until an odd number appears. Probability of all evens before first odd?}
            
            \item[Level 3:] Advanced beginner problems requiring creative thinking and synthesis of multiple concepts. \textit{e.g., Triangle $ABC$ with $AB=50, AC=10$ area 120... find Area of quadrilateral $FDBG$.}
            
            \item[Level 4:] Intermediate-level problems involving sequences, functional iteration, or systems of equations. \textit{e.g., Sequence $x_{n+1}=\frac{x_n^2+5x_n+4}{x_n+6}$, find least $m$ s.t. $x_m \leq 4+\frac{1}{2^{20}}$.}
            
            \item[Level 5:] High-intermediate problems, including simple proof-based Olympiad-style questions. \textit{e.g., Find triples $(a, b, c)$ s.t. $a+b+c=\sum \frac{1}{a}$ and $a^2+b^2+c^2=\sum \frac{1}{a^2}$.}
            
            \item[Level 6:] Advanced high-school level to introductory Olympiad level. Requires integration of geometry, algebra, and number theory. \textit{e.g., Acute $\triangle ABC$ with orthocenter $H$... Area is $m\sqrt{n}$.}
            
            \item[Level 7:] Mid-tier Olympiad level. Requires technical knowledge, combinatorial constructions, or graph-theoretic reasoning. \textit{e.g., Prove existence of a balanced set of $n$ points and determine centre-free conditions.}
            
            \item[Level 8:] High-level Olympiad problems. Demands clever constructions or deep insights. \textit{e.g., Splitting coins of denomination $\frac{1}{n}$ with total value $99+\frac{1}{2}$ into groups $\le 1$.}
            
            \item[Level 9:] Expert Olympiad level. Involves advanced number theory, algebraic combinatorics, or functional constraints. \textit{e.g., Primes on a circle where neighbor product is $x^2+x+k$.}
            
            \item[Level 10:] Historically extreme problems. Exceedingly tedious, technical, or rare. Solved by very few worldwide. \textit{e.g., Existence of line separating $n$ points with margin $\ge c n^{-1/3}$.}
        \end{description}

        \vspace{0.4em}
        \textbf{\#\# Generation Rules (STRICT)}
        \begin{enumerate}[nosep, leftmargin=1.5em]
            \item \textbf{Preserve type \& relationship:} Maintain original problem type (e.g., Geometry) and core logic (e.g., area ratios).
            \item \textbf{To reduce difficulty:} Use friendlier numbers, reduce logical steps, simplify story context, or remove auxiliary constructions.
            \item \textbf{Clarity:} The output must be self-contained, have a unique well-defined answer, and use natural mathematical English.
        \end{enumerate}

        \vspace{0.4em}
        \textbf{\#\# Input Data} \\
        \textit{Problem:} \{query\} $|$ \textit{Answer:} \{answer\} $|$ \textit{Level:} \{difficulty\} $|$ \textit{Type:} \{type\}

        \vspace{0.4em}
        \textbf{\#\# Output Instruction} \\
        Output \textbf{ONLY} the new problem statement. No labels, no explanations, no extra text. \\
        \textbf{Format:} [Your generated problem statement]
    \end{tcolorbox}
    \caption{The prompt for difficulty-reduction.}
    \label{tab:dif_reduction_prompt}
\end{table*}

\begin{table*}[ht!]
    \centering
    \scriptsize 
    \begin{tcolorbox}[
        sharp corners,
        colback=white,
        colframe=black!85,
        fonttitle=\bfseries,
        title=Prompt for Difficulty Increasing Agent (N $\to$ N+1),
        boxrule=0.7pt,
        left=5pt, right=5pt, top=5pt, bottom=5pt
    ]
        \textbf{\#\# Task Definition} \\
        You are an expert mathematics problem designer. Based on the input, calibrate and generate a new problem:
        \begin{itemize}[nosep, leftmargin=1.5em]
            \item If original difficulty $D < 10$, generate a new problem at level $D+1$.
            \item If $D = 10$, generate a similar Level 10 problem with different structures/objects for expansion.
        \end{itemize}

        \vspace{0.4em}
        \textbf{\#\# Difficulty Calibration Standard (Levels 1--10)}
        \begin{description}[style=unboxed, leftmargin=0pt, font=\bfseries, nosep, itemsep=2pt]
            \item[Level 1:] Problems strictly for beginners, covering elementary or early middle school content. Involves standard techniques and traditional word problems. \textit{e.g., How many integer values of $x$ satisfy $|x| < 3\pi$?}
            
            \item[Level 2:] For motivated beginners. Involves harder instructions, complex word problems, and requires multi-step reasoning beyond basic techniques. \textit{e.g., A fair 6-sided die is repeatedly rolled until an odd number appears. Probability of all evens before first odd?}
            
            \item[Level 3:] Advanced beginner problems requiring creative thinking and synthesis of multiple concepts. \textit{e.g., Triangle $ABC$ with $AB=50, AC=10$ area 120... find Area of quadrilateral $FDBG$.}
            
            \item[Level 4:] Intermediate-level problems involving sequences, functional iteration, or systems of equations. \textit{e.g., Sequence $x_{n+1}=\frac{x_n^2+5x_n+4}{x_n+6}$, find least $m$ s.t. $x_m \leq 4+\frac{1}{2^{20}}$.}
            
            \item[Level 5:] High-intermediate problems, including simple proof-based Olympiad-style questions. \textit{e.g., Find triples $(a, b, c)$ s.t. $a+b+c=\sum \frac{1}{a}$ and $a^2+b^2+c^2=\sum \frac{1}{a^2}$.}
            
            \item[Level 6:] Advanced high-school level to introductory Olympiad level. Requires integration of geometry, algebra, and number theory. \textit{e.g., Acute $\triangle ABC$ with orthocenter $H$... Area is $m\sqrt{n}$.}
            
            \item[Level 7:] Mid-tier Olympiad level. Requires technical knowledge, combinatorial constructions, or graph-theoretic reasoning. \textit{e.g., Prove existence of a balanced set of $n$ points and determine centre-free conditions.}
            
            \item[Level 8:] High-level Olympiad problems. Demands clever constructions or deep insights. \textit{e.g., Splitting coins of denomination $\frac{1}{n}$ with total value $99+\frac{1}{2}$ into groups $\le 1$.}
            
            \item[Level 9:] Expert Olympiad level. Involves advanced number theory, algebraic combinatorics, or functional constraints. \textit{e.g., Primes on a circle where neighbor product is $x^2+x+k$.}
            
            \item[Level 10:] Historically extreme problems. Exceedingly tedious, technical, or rare. Solved by very few worldwide. \textit{e.g., Existence of line separating $n$ points with margin $\ge c n^{-1/3}$.}
        \end{description}

        \vspace{0.4em}
        \textbf{\#\# Generation Rules (STRICT)}
        \begin{enumerate}[nosep, leftmargin=1.5em]
            \item \textbf{Preserve type \& structure:} Maintain core type (e.g., Number Theory) and structure (e.g., recurrence).
            \item \textbf{To increase difficulty ($N \to N+1$):} Introduce one layer of abstraction; add \textbf{exactly one} reasoning step (e.g., a lemma or substitution); or incorporate a related concept (e.g., add modular arithmetic).
            \item \textbf{Constraint:} Do \textbf{not} jump to tools typical of Level $N+2$ or beyond. 
            \item \textbf{Clarity:} Phrased rigorously in math English with a unique answer or precise proof statement.
        \end{enumerate}

        \vspace{0.4em}
        \textbf{\#\# Input Data} \\
        \textit{Problem:} \{query\} $|$ \textit{Answer:} \{answer\} $|$ \textit{Level:} \{difficulty\} $|$ \textit{Type:} \{type\}

        \vspace{0.4em}
        \textbf{\#\# Output Instruction} \\
        Output \textbf{ONLY} the new problem statement. \\
        \textbf{Format:} [Your generated problem statement]
    \end{tcolorbox}
    \caption{The structured prompt for the Problem Increasing Agent, calibrated for Level $N+1$ difficulty escalation.}
    \label{tab:increaser_prompt}
\end{table*}

\begin{table*}[ht!]
    \centering
    \scriptsize
    \begin{tcolorbox}[
        sharp corners,
        colback=gray!2,
        colframe=black!85,
        fonttitle=\bfseries,
        title=Prompt for diversity enhancement agent,
        boxrule=0.7pt,
        left=5pt, right=5pt, top=5pt, bottom=5pt
    ]
        \textbf{\#\# Role \& Objective} \\
        You are an expert mathematics question designer. Generate a new problem by evolving the input problem to \textbf{Level \{target\_level\}} while explicitly incorporating the \textbf{[\{type\}]} concept.

        \vspace{0.4em}
        \textbf{\#\# Official 10-Level Difficulty Scale (Reference)}
        \begin{description}[style=unboxed, leftmargin=0pt, font=\bfseries, nosep, itemsep=2pt]
            \item[Level 1:] Problems strictly for beginners, covering elementary or early middle school content. Involves standard techniques and traditional word problems. \textit{e.g., How many integer values of $x$ satisfy $|x| < 3\pi$?}
            
            \item[Level 2:] For motivated beginners. Involves harder instructions, complex word problems, and requires multi-step reasoning beyond basic techniques. \textit{e.g., A fair 6-sided die is repeatedly rolled until an odd number appears. Probability of all evens before first odd?}
            
            \item[Level 3:] Advanced beginner problems requiring creative thinking and synthesis of multiple concepts. \textit{e.g., Triangle $ABC$ with $AB=50, AC=10$ area 120... find Area of quadrilateral $FDBG$.}
            
            \item[Level 4:] Intermediate-level problems involving sequences, functional iteration, or systems of equations. \textit{e.g., Sequence $x_{n+1}=\frac{x_n^2+5x_n+4}{x_n+6}$, find least $m$ s.t. $x_m \leq 4+\frac{1}{2^{20}}$.}
            
            \item[Level 5:] High-intermediate problems, including simple proof-based Olympiad-style questions. \textit{e.g., Find triples $(a, b, c)$ s.t. $a+b+c=\sum \frac{1}{a}$ and $a^2+b^2+c^2=\sum \frac{1}{a^2}$.}
            
            \item[Level 6:] Advanced high-school level to introductory Olympiad level. Requires integration of geometry, algebra, and number theory. \textit{e.g., Acute $\triangle ABC$ with orthocenter $H$... Area is $m\sqrt{n}$.}
            
            \item[Level 7:] Mid-tier Olympiad level. Requires technical knowledge, combinatorial constructions, or graph-theoretic reasoning. \textit{e.g., Prove existence of a balanced set of $n$ points and determine centre-free conditions.}
            
            \item[Level 8:] High-level Olympiad problems. Demands clever constructions or deep insights. \textit{e.g., Splitting coins of denomination $\frac{1}{n}$ with total value $99+\frac{1}{2}$ into groups $\le 1$.}
            
            \item[Level 9:] Expert Olympiad level. Involves advanced number theory, algebraic combinatorics, or functional constraints. \textit{e.g., Primes on a circle where neighbor product is $x^2+x+k$.}
            
            \item[Level 10:] Historically extreme problems. Exceedingly tedious, technical, or rare. Solved by very few worldwide. \textit{e.g., Existence of line separating $n$ points with margin $\ge c n^{-1/3}$.}
        \end{description}

        \vspace{0.4em}
        \textbf{\#\# Evolution Guidelines}
        \begin{enumerate}[nosep, leftmargin=1.5em]
            \item \textbf{Type Injection:} Explicitly integrate the [\textbf{\{type\}}] mathematical concept into the solution path.
            \item \textbf{Context Preservation:} Preserve the core scenario and solution objectives of the original problem as much as possible.
            \item \textbf{Precision Calibration:} Adjust difficulty to exactly \textbf{Level \{target\_level\}}. Only increase/decrease complexity as justified by the scale.
            \item \textbf{Constraints:} Do not introduce case splits, hidden variables, or multi-concept chaining beyond the norms of the target level.
        \end{enumerate}

        \vspace{0.4em}
        \textbf{\#\# Input Specification} \\
        \textbf{Original Problem:} \{query\} $|$ \textbf{Original Answer:} \{answer\} \\
        \textbf{Current Level:} \{original\_level\} $|$ \textbf{Target Level:} \{target\_level\} $|$ \textbf{Target Type:} \{type\}

        \vspace{0.4em}
        \textbf{\#\# Output Format (STRICT)} \\
        Output \textbf{ONLY} the new problem. No explanations, no extra text. \\
        \textbf{Format:} \texttt{query: [Your generated problem]}
    \end{tcolorbox}
    \caption{The structured prompt for diversity enhancement.}
    \label{tab:add_type_prompt}
\end{table*}

\begin{table*}[ht!]
    \centering
    \scriptsize 
    \begin{tcolorbox}[
        sharp corners,
        colback=gray!5,
        colframe=black!85,
        fonttitle=\bfseries,
        title=Prompt for reverse generation agent,
        boxrule=0.7pt,
        left=5pt, right=5pt, top=5pt, bottom=5pt
    ]
        \textbf{\#\# Role \& Core Logic} \\
        You are a precise math problem generator. Given a problem and its answer, create a \textbf{reversed} version of the problem by:
        \begin{itemize}[nosep, leftmargin=1.5em]
            \item Converting one of the original input conditions into the unknown.
            \item Using the original answer (or a value derived from it) as a new given condition.
        \end{itemize}

        \vspace{0.4em}
        \textbf{\#\# Requirements (STRICT)}
        \begin{itemize}[nosep, leftmargin=1.5em]
            \item \textbf{Consistency:} Keep the same mathematical relationship, formula, or logic as the original problem.
            \item \textbf{Variable Swap:} Only swap the role of one variable; treat the original answer as a known quantity.
            \item \textbf{Complexity Control:} Do not change problem type, add new concepts, or increase difficulty.
            \item \textbf{Integrity:} The reversed problem must have a unique, correct solution and use standard mathematical wording.
        \end{itemize}

        \vspace{0.4em}
        \textbf{\#\# Negative Constraints (Forbidden Content)}
        Do \textbf{NOT} include any extra text, such as:
        \begin{itemize}[nosep, leftmargin=1.5em]
            \item Meta-language (e.g., "reversed", "original", "Here is the problem").
            \item Formatting tags like LaTeX \texttt{\textbackslash boxed\{\}} or any specific answer formatting.
            \item Explanations, comments, or any text outside the problem statement itself.
        \end{itemize}

        \vspace{0.4em}
        \textbf{\#\# Input Specification} \\
        \textbf{Original Problem:} \{query\} $|$ \textbf{Original Answer:} \{answer\}

        \vspace{0.4em}
        \textbf{\#\# Output Instruction} \\
        Output \textbf{ONLY} the generated problem statement. \\
        \textbf{Format:} \texttt{[your generated problem]}
    \end{tcolorbox}
    \caption{The structured prompt for generating inverse mathematical problems by swapping known and unknown variables.}
    \label{tab:reverse_prompt}
\end{table*}

\begin{table*}[ht!]
    \centering
    \scriptsize 
    \begin{tcolorbox}[
        sharp corners,
        colback=gray!5,
        colframe=black!85,
        fonttitle=\bfseries,
        title=Prompt for Mathematical Solution Verification Agent,
        boxrule=0.7pt,
        left=5pt, right=5pt, top=5pt, bottom=5pt
    ]
        \textbf{\#\# Role \& Core Task} \\
        You are an expert AI assistant specializing in verifying mathematical solutions for accuracy. Your task is to \textbf{evaluate only the provided answer} to the given mathematical query and determine if it is \textbf{correct}.

        \vspace{0.4em}
        \textbf{\#\# Verification Instructions}
        \begin{itemize}[nosep, leftmargin=1.5em]
            \item \textbf{Comprehensive Check:} Verify whether the answer is \textbf{entirely correct}, including reasoning (if present), final results, calculations, and mathematical logic.
            \item \textbf{Final Result Accuracy:} If the answer contains a final result (e.g., a number, expression, or boxed value), it must be \textbf{exactly correct}.
            \item \textbf{Failure Cases:} If the provided answer is missing, empty, or fails to produce an expected final result, treat it as \textbf{incorrect}.
        \end{itemize}

        \vspace{0.4em}
        \textbf{\#\# Strict Output Constraint}
        \begin{itemize}[nosep, leftmargin=1.5em]
            \item \textbf{Output Format:} Output \textbf{ONLY} one word: \textbf{"yes"} if the answer is correct, \textbf{"no"} if it is incorrect or contains any error.
            \item \textbf{Forbidden Content:} Do \textbf{not} output any explanation, punctuation, extra text, or formatting.
        \end{itemize}

        \vspace{0.4em}
        \textbf{\#\# Input Data} \\
        \textbf{Query:} \{query\} \\
        \textbf{Answer:} \{answer\}

        \vspace{0.4em}
        \textbf{\#\# Execution Instruction} \\
        Analyze the query and answer according to the logic above. Output exactly one word: \textbf{yes} or \textbf{no}.
    \end{tcolorbox}
    \caption{The structured prompt for the Verification Agent, designed for binary correctness evaluation of mathematical answers.}
    \label{tab:verification_prompt}
\end{table*}

\section{Case Studies}
\begin{itemize}
    \item \textbf{Original Case}: Table~\ref{table:original_tab}.

    \item \textbf{Reduce Difficulty}:
    Table~\ref{table:reduce_difficulty}.

    \item \textbf{Add Difficulty}: Table~\ref{table:add_difficulty}.

    \item \textbf{Reverse Generation}:
    Table~\ref{table:reverse_generate}.

    \item \textbf{Add Diversity}:
    Table~\ref{table:add_diversity}.

\end{itemize}

\begin{table*}[ht!]
    \centering
    \scriptsize
    \begin{tcolorbox}[
        sharp corners,
        colback=gray!5,
        colframe=black!85,
        fonttitle=\bfseries,
        title=Original Problem and Solution,
        boxrule=0.7pt,
        left=5pt, right=5pt, top=5pt, bottom=5pt
    ]
        \textbf{Query:} \\
        If there is a positive real number \( x \) such that \( \cos(\arctan(x)) = x \), what is the value of \( x^2 \)?

        \vspace{0.6em}
        \textbf{Answer:} \\
        The given equation is \( \cos(\arctan(x)) = x \) for \( x > 0 \). Using the identity
        \[
        \cos(\arctan(x)) = \frac{1}{\sqrt{1+x^2}},
        \]
        the equation becomes:
        \[
        \frac{1}{\sqrt{1+x^2}} = x.
        \]

        Squaring both sides (since \( x > 0 \)) yields:
        \[
        \frac{1}{1+x^2} = x^2.
        \]

        Multiplying both sides by \( 1+x^2 \) gives:
        \[
        1 = x^2(1+x^2) = x^2 + x^4.
        \]

        Rearranging into a quadratic in \( y = x^2 \):
        \[
        y^2 + y - 1 = 0.
        \]

        Solving for \( y \):
        \[
        y = \frac{-1 \pm \sqrt{5}}{2}.
        \]

        Since \( y = x^2 > 0 \), we take the positive root:
        \[
        x^2 = \frac{\sqrt{5} - 1}{2}.
        \]

        Thus, the value of \( x^2 \) is \( \dfrac{\sqrt{5}-1}{2} \).

        \[
        \boxed{\dfrac{\sqrt{5}-1}{2}}
        \]
    \end{tcolorbox}
    \caption{The original question and its complete solution.}
    \label{table:original_tab}
\end{table*}

\begin{table*}[ht!]
    \centering
    \scriptsize
    \begin{tcolorbox}[
        sharp corners,
        colback=gray!5,
        colframe=black!85,
        fonttitle=\bfseries,
        title=Reduce difficulty by difficulty reduction agent,
        boxrule=0.7pt,
        left=5pt, right=5pt, top=5pt, bottom=5pt
    ]
        \textbf{Query:} \\
        If there is a positive real number \( x \) such that one-half of \( x \) plus 6 equals \( x \), what is the value of \( x^2 \)?

        \vspace{0.6em}
        \textbf{Answer:} \\
        The equation described is \( \frac{1}{2}x + 6 = x \). Solving for \( x \):
        \[
        \begin{aligned}
            \frac{1}{2}x + 6 &= x \\
            6 &= x - \frac{1}{2}x \\
            6 &= \frac{1}{2}x \\
            x &= 12.
        \end{aligned}
        \]

        Then \( x^2 = 12^2 = 144 \).

        \[
        \boxed{144}
        \]
    \end{tcolorbox}
    \caption{The case reducing the difficulty from the original question.}
    \label{table:reduce_difficulty}
\end{table*}

\begin{table*}[ht!]
    \centering
    \scriptsize
    \begin{tcolorbox}[
        sharp corners,
        colback=gray!5,
        colframe=black!85,
        fonttitle=\bfseries,
        title=Add difficulty by Difficulty Increasing Agent,
        boxrule=0.7pt,
        left=5pt, right=5pt, top=5pt, bottom=5pt
    ]
        \textbf{Query:} \\
        If there is a positive real number \( x \) such that \( \cos(\arctan(x)) = \tan(\arcsin(x)) \), what is the value of \( x^2 \)?

        \vspace{0.6em}
        \textbf{Answer:} \\
        The given equation is \( \cos(\arctan(x)) = \tan(\arcsin(x)) \) for \( x > 0 \).

        \begin{itemize}
            \item Left-hand side: Let \( \theta = \arctan(x) \), so \( \tan \theta = x \). Then \( \cos \theta = \dfrac{1}{\sqrt{1+x^2}} \).
            \item Right-hand side: Let \( \phi = \arcsin(x) \), so \( \sin \phi = x \). Then \( \tan \phi = \dfrac{x}{\sqrt{1-x^2}} \), valid for \( 0 < x < 1 \).
        \end{itemize}

        Thus the equation becomes:
        \[
        \frac{1}{\sqrt{1+x^2}} = \frac{x}{\sqrt{1-x^2}}.
        \]

        Squaring both sides yields:
        \[
        \frac{1}{1+x^2} = \frac{x^2}{1-x^2}.
        \]

        Cross-multiplying:
        \[
        1 - x^2 = x^2(1 + x^2) = x^2 + x^4.
        \]

        Rearranging:
        \[
        x^4 + 2x^2 - 1 = 0.
        \]

        Substitute \( u = x^2 \):
        \[
        u^2 + 2u - 1 = 0.
        \]

        Solving:
        \[
        u = \frac{-2 \pm \sqrt{4 + 4}}{2} = \frac{-2 \pm 2\sqrt{2}}{2} = -1 \pm \sqrt{2}.
        \]

        Since \( u > 0 \), we have \( u = \sqrt{2} - 1 \). Verification shows that this satisfies the original equation.

        Thus, \( x^2 = \sqrt{2} - 1 \).

        \[
        \boxed{\sqrt{2} - 1}
        \]
    \end{tcolorbox}
    \caption{The case increasing the difficulty from the original question.}
    \label{table:add_difficulty}
\end{table*}

\begin{table*}[ht!]
    \centering
    \scriptsize
    \begin{tcolorbox}[
        sharp corners,
        colback=gray!5,
        colframe=black!85,
        fonttitle=\bfseries,
        title=Reverse Generation,
        boxrule=0.7pt,
        left=5pt, right=5pt, top=5pt, bottom=5pt
    ]
        \textbf{Query:} \\
        For a positive real number \( x \) such that \( x^2 = \dfrac{\sqrt{5} - 1}{2} \), find the value of \( \cos(\arctan(x)) \).

        \vspace{0.6em}
        \textbf{Answer:} \\
        Given \( x > 0 \) and \( x^2 = \dfrac{\sqrt{5} - 1}{2} \), we need to find \( \cos(\arctan(x)) \).

        Let \( \theta = \arctan(x) \), so \( \tan(\theta) = x \). Since \( x > 0 \), \( \theta \in (0, \frac{\pi}{2}) \). Then,
        \[
        \cos(\theta) = \frac{1}{\sqrt{1 + \tan^2(\theta)}} = \frac{1}{\sqrt{1 + x^2}}.
        \]

        Substitute \( x^2 = \dfrac{\sqrt{5} - 1}{2} \):
        \[
        1 + x^2 = 1 + \frac{\sqrt{5} - 1}{2} = \frac{\sqrt{5} + 1}{2}.
        \]

        Thus,
        \[
        \cos(\arctan(x)) = \frac{1}{\sqrt{\frac{\sqrt{5} + 1}{2}}} = \sqrt{\frac{2}{\sqrt{5} + 1}}.
        \]

        Rationalize the expression:
        \[
        \sqrt{\frac{2}{\sqrt{5} + 1}} = \sqrt{\frac{2(\sqrt{5} - 1)}{(\sqrt{5} + 1)(\sqrt{5} - 1)}} = \sqrt{\frac{2(\sqrt{5} - 1)}{4}} = \sqrt{\frac{\sqrt{5} - 1}{2}} = x.
        \]

        Therefore, the value is \( \sqrt{\dfrac{\sqrt{5} - 1}{2}} \).

        \[
        \boxed{\sqrt{\dfrac{\sqrt{5}-1}{2}}}
        \]
    \end{tcolorbox}
    \caption{A reverse case derived from the original problem.}
    \label{table:reverse_generate}
\end{table*}

\begin{table*}[ht!]
    \centering
    \scriptsize
    \begin{tcolorbox}[
        sharp corners,
        colback=gray!5,
        colframe=black!85,
        fonttitle=\bfseries,
        title=Add diversity from Prealgebra to Geometry,
        boxrule=0.7pt,
        left=5pt, right=5pt, top=5pt, bottom=5pt
    ]
        \textbf{Query:} \\
        In right triangle \( ABC \) with right angle at \( B \), let \( AB = 1 \) and \( BC = x \). Let \( D \) be the foot of the altitude from \( B \) to hypotenuse \( AC \). If \( AD = BC \), find \( x^2 \).

        \vspace{0.6em}
        \textbf{Answer:} \\
        In right triangle \( ABC \) with right angle at \( B \), let \( AB = 1 \) and \( BC = x \). The foot of the altitude from \( B \) to hypotenuse \( AC \) is \( D \), and it is given that \( AD = BC = x \).

        Using similar triangles, \( AD = \dfrac{AB^2}{AC} = \dfrac{1}{\sqrt{1+x^2}} \). Setting this equal to \( x \) gives:
        \[
        \frac{1}{\sqrt{1+x^2}} = x.
        \]

        Squaring both sides yields:
        \[
        \frac{1}{1+x^2} = x^2.
        \]

        Multiply both sides by \( 1+x^2 \):
        \[
        1 = x^2(1+x^2) = x^2 + x^4.
        \]

        Rearranging:
        \[
        x^4 + x^2 - 1 = 0.
        \]

        Let \( u = x^2 \), then \( u^2 + u - 1 = 0 \). Solving the quadratic:
        \[
        u = \frac{-1 \pm \sqrt{1+4}}{2} = \frac{-1 \pm \sqrt{5}}{2}.
        \]

        Since \( x^2 > 0 \), we take the positive root:
        \[
        x^2 = \frac{\sqrt{5} - 1}{2}.
        \]

        \[
        \boxed{\dfrac{\sqrt{5}-1}{2}}
        \]
    \end{tcolorbox}
    \caption{A geometric variant derived from an algebraic original problem.}
    \label{table:add_diversity}
\end{table*}

\renewcommand{\algorithmicrequire}{\textbf{Input:}}  
\renewcommand{\algorithmicensure}{\textbf{Output:}}  

\begin{algorithm*}[t]
\caption{Adaptive Bidirectional Curriculum Generation}
\label{alg:bidirectional_curriculum}
\setlength{\baselineskip}{0.95\baselineskip}
\begin{algorithmic}[1]
\Require 
    Initial seed dataset $\mathcal{D}_{\text{seed}}$;
    Initial student model $\pi_{\theta_0}$;
    Maximum iterations $T$;
    Difficulty mapping function $\Phi(\cdot)$;
    Transformation operators $\mathcal{G} = \{\mathcal{G}_{\text{red}}, \mathcal{G}_{\text{rev}}, \mathcal{G}_{\text{inc}}, \mathcal{G}_{\text{div}}\}$.

\Ensure Optimized Student Model $\pi_{\theta_T}$.

    \State \textbf{// Stage 1: Seed Initialization (Section~\ref{subsec:initialize})}
    \State $\mathcal{D}_{\text{val}}^{(0)} \leftarrow \text{StratifiedSample}(\mathcal{D}_{\text{seed}}, N=200)$ based on subjects and $\Phi(p)$
    \State Initialize failure counter $\text{ind}_{\text{err}}(p) \leftarrow 0$ for all $p \in \mathcal{D}_{\text{val}}^{(0)}$
    \State $\mathcal{D}_{\text{train}}^{(0)} \leftarrow \emptyset$

    \Statex
    \State \textbf{// Main Iterative Curriculum Loop}
    \For{$t = 0$ to $T-1$}
        \State \textbf{// Stage 2: Diagnostic Evaluation (Section~\ref{subsec:diag})}
        \State $S_{\text{hard}}^{(t)}, S_{\text{easy}}^{(t)} \leftarrow \text{Partition}(\mathcal{D}_{\text{val}}^{(t)})$ based on $\mathbbm{1}_{\text{corr}}(p, \pi_{\theta_t})$
        \For{$p \in S_{\text{hard}}^{(t)}$}
            \State $\text{ind}_{\text{err}}(p) \leftarrow \text{ind}_{\text{err}}(p) + 1$ \Comment{Update Error Retention Policy}
        \EndFor

        \Statex
        \State \textbf{// Stage 3: Multi-Agent Data Generation (Section~\ref{subsec:multiagent})}
        \State $\mathcal{D}_{\text{remedy}}^{(t)} \leftarrow \emptyset$, $\mathcal{D}_{\text{adv}}^{(t)} \leftarrow \emptyset$
        
        \For{$p \in S_{\text{hard}}^{(t)}$} \Comment{Downward Adjustment}
            \State Generate $p' \leftarrow \{\mathcal{G}_{\text{red}}(p), \mathcal{G}_{\text{rev}}(p)\}$ 
            \State $\mathcal{D}_{\text{remedy}}^{(t)} \leftarrow \mathcal{D}_{\text{remedy}}^{(t)} \cup \{p' \mid \text{CheckFormat}(p') \land \text{Verify}(p')\}$
        \EndFor
        
        \For{$p \in S_{\text{easy}}^{(t)}$} \Comment{Upward Expansion}
            \State Generate $p'' \leftarrow \{\mathcal{G}_{\text{inc}}(p), \mathcal{G}_{\text{div}}(p)\}$
            \State $\mathcal{D}_{\text{adv}}^{(t)} \leftarrow \mathcal{D}_{\text{adv}}^{(t)} \cup \{p'' \mid \text{CheckFormat}(p'') \land \text{Verify}(p'')\}$
        \EndFor

        \Statex
        \State \textbf{// Stage 4: Curriculum Co-evolution (Section~\ref{subsec:coevolution})}
        \State \textbf{1. Persistent Failures:} $S_{\text{stubborn}} \leftarrow \{p \in S_{\text{hard}}^{(t)} \mid \text{ind}_{\text{err}}(p) > 3\}$
        \State \textbf{2. Training Set Update:} $\mathcal{D}_{\text{train}}^{(t+1)} \leftarrow \mathcal{D}_{\text{remedy}}^{(t)} \cup S_{\text{stubborn}}$
        \State \textbf{3. Validation Set Renewal:} $\mathcal{D}_{\text{val}}^{(t+1)} \leftarrow \{p \in S_{\text{hard}}^{(t)} \mid \text{ind}_{\text{err}}(p) \le 3\} \cup \mathcal{D}_{\text{adv}}^{(t)}$

        \Statex
        \State \textbf{// Model Optimization}
        \State $\theta_{t+1} \leftarrow \arg \min_{\theta} \mathbb{E}_{(x,y) \sim \mathcal{D}_{\text{train}}^{(t+1)}} [-\log \pi_{\theta}(y|x)]$ \Comment{Supervised Fine-Tuning}
    \EndFor

    \State \Return $\pi_{\theta_T}$
\end{algorithmic}
\end{algorithm*}

\end{document}